\documentclass[10pt,journal,compsoc]{IEEEtran}

\usepackage{amsmath,amssymb,amsthm,mathtools}
\usepackage{graphicx}
\usepackage{booktabs}
\usepackage{array}
\usepackage{multirow}
\usepackage{cite}
\usepackage{url}
\usepackage{enumitem}
\usepackage{balance}
\usepackage[T1]{fontenc}
\usepackage{lmodern}

\newtheorem{theorem}{Theorem}
\newtheorem{proposition}{Proposition}
\newtheorem{corollary}{Corollary}
\newtheorem{lemma}{Lemma}
\theoremstyle{definition}
\newtheorem{definition}{Definition}
\theoremstyle{remark}
\newtheorem{remark}{Remark}

\newcommand{\OT}{\mathrm{OT}}
\newcommand{\dist}{\mathrm{dist}}
\newcommand{\rank}{\mathrm{rank}}
\newcommand{\area}{\mathrm{area}}
\newcommand{\osc}{\mathrm{osc}}
\newcommand{\R}{\mathbb{R}}
\newcommand{\one}{\mathbf{1}}
\newcommand{\pos}[1]{[#1]_+}
\newcommand{\NW}{\mathrm{NW}}
\newcommand{\Mcal}{\mathcal{M}}
\newcommand{\Dcal}{\mathcal{D}}
\newcommand{\Bcal}{\mathcal{B}}
\newcommand{\PiMix}{\Pi_{\mathrm{mix}}}

\title{From Cone Geometry to Monge Structure: Local Defects and Global Regret in High-Dimensional Optimal Transport}

\author{Lei Luo and Jian Yang%
\thanks{Lei Luo and Jian Yang are with PCA Lab, Key Lab of Intelligent Perception and Systems for High-Dimensional Information of Ministry of Education, School of Computer Science and Engineering, Nanjing University of Science and Technology, Nanjing, China. E-mail: \{cslluo, csjyang\}@njust.edu.cn. Lei Luo and Jian Yang are the corresponding authors.}}

\begin{document}
\maketitle

\begin{abstract}
Exact optimal transport in high dimensions becomes tractable only when additional structure selects the optimal coupling. Classical Monge transportation solves the discrete problem once a cost array is known to be Monge, but it does not explain when geometry in the original feature space creates that structure. We characterize this regime for cone-induced orders and squared Mahalanobis costs: $M$-acuteness of the cone is equivalent to the Monge property of every cross-cost matrix generated by two cone chains, after which classical northwest-corner optimality gives exact transport. When the inequalities fail, we derive an exact variational representation of northwest-corner regret over a marginal-dependent cumulative-deficit polytope. Its Fr\'echet box relaxation yields a projection-free, marginal-weighted local-defect bound; we characterize when the relaxation is exact and prove that it is worst-case sharp when only local defect budgets are known. We further show that the worst-case defect patterns can be realized by genuine squared-Euclidean cross-costs. Finally, Mahalanobis distances of observed increments to the cone control the local defects, while a strict geometric margin gives an explicit radius of exact stability. The resulting theory links high-dimensional ground-space geometry to exact Monge transport and quantifies the global loss induced by local departures from ordered structure.
\end{abstract}


\section{Introduction}
Optimal transport (OT) compares probability measures through the minimum cost of moving mass between them. Its geometric meaning has made it a basic tool in statistics, imaging, computer vision, representation learning, and modern generative modeling \cite{r1,r2,r3}. Yet the same global coupling that makes OT expressive also makes exact computation difficult. For empirical measures, one typically optimizes over a transportation polytope whose dimension grows quadratically with the number of support points. Entropic regularization and Sinkhorn scaling alleviate this burden, but they remain iterative numerical procedures and alter the original objective \cite{r3,r4,r5}.

The one-dimensional case is different for a structural reason. On the line, order and convex transport cost agree: if $x_1\le x_2$ and $y_1\le y_2$, crossing the assignments cannot improve the cost. The optimal coupling is therefore the monotone rearrangement. This familiar fact suggests a broader question. The source of analytic tractability is not merely low ambient dimension; it is the existence of an order that is compatible with the cost.

There is a classical discrete counterpart of the same idea. A matrix $C\in\R^{m\times n}$ is Monge when
\begin{equation}
C_{ij}+C_{k\ell}\le C_{i\ell}+C_{kj},\qquad i<k,\ j<\ell.
\end{equation}
For transportation problems with a Monge cost matrix, the northwest-corner solution is optimal for every choice of marginals. This theory is classical and well developed \cite{r21,r22,r23}. Later work studied recognition of Monge and permuted-Monge arrays, statistical estimation of Monge structure, and transportation models under Monge costs \cite{r24,r25,r29}. Thus, once the ordered matrix structure is already known, the transport problem is essentially solved.

The question in this paper comes one step earlier:
\begin{quote}
\emph{When does high-dimensional ground-space geometry force the source--target cost matrix to be Monge?}
\end{quote}
This distinction matters. Matrix recognition works after a cross-cost array has been formed and searches for, or verifies, an order of its rows and columns. We instead ask whether a geometric relation in the original feature space implies Monge structure for an entire family of transport instances before any particular matrix is constructed.

We study this question through convex cones. A closed convex cone $K\subset\R^d$ induces the partial order
\begin{equation}
x\preceq_K y \quad\Longleftrightarrow\quad y-x\in K.
\end{equation}
For squared Mahalanobis cost
\begin{equation}
c_M(x,y)=(x-y)^\top M(x-y),\qquad M\succeq 0,
\end{equation}
we prove that the single geometric condition
\begin{equation}
u^\top Mv\ge 0,\qquad \forall u,v\in K,
\end{equation}
is equivalent to a uniform Monge statement: every pair of finite $K$-chains generates a Monge cross-cost matrix. The proof is elementary but the equivalence is consequential. It turns a property of a potentially large cost matrix into a reusable condition on the ground geometry. For a finitely generated cone $K=A\R_+^r$, the condition reduces to entrywise nonnegativity of $A^\top M A$.

\begin{figure}[t]
\centering
\includegraphics[width=\columnwidth]{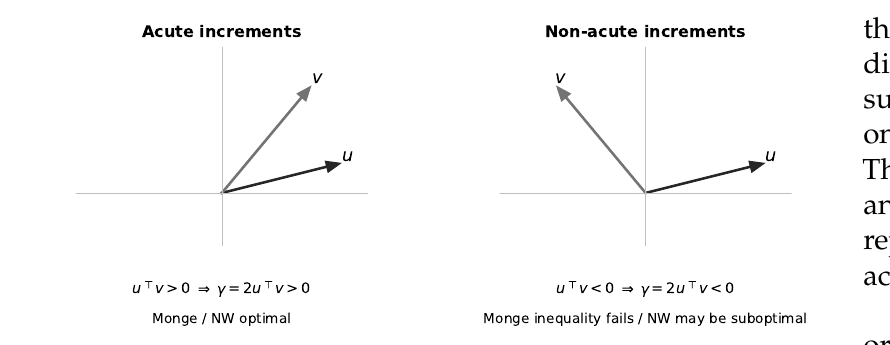}
\caption{A $2\times2$ preview of the geometry-to-Monge mechanism for $M=I$. With source increment $u=x_2-x_1$ and target increment $v=y_2-y_1$, the adjacent Monge margin is $\gamma=2u^\top v$. Acute increments therefore induce the Monge inequality and classical northwest-corner optimality; a negative inner product reverses the inequality and the monotone plan can become suboptimal.}
\label{fig:preview}
\end{figure}

Once Monge structure has been induced, we deliberately return to the classical theory. The cumulative-overlap or northwest-corner plan is not presented as a new algorithm. Its role here is to show that a geometric condition on the original space can make the exact Kantorovich problem collapse to the classical monotone solution without projection to one-dimensional slices and without replacing the ground metric by a tree.

Exact Monge inequalities, however, are too brittle to describe most measured data. The second part of the paper therefore asks a quantitative question: if the order is only approximately compatible with the cost, how much can the monotone plan lose? Let $\Mcal_{m,n}$ be the cone of $m\times n$ Monge matrices and define
\begin{equation}
\delta_M(C)=\inf_{\widetilde C\in\Mcal_{m,n}}\|C-\widetilde C\|_\infty.
\end{equation}
If $\pi_{\NW}$ denotes the northwest-corner plan for the given marginals, then
\begin{equation}
0\le \langle C,\pi_{\NW}\rangle-\OT_C\le 2\delta_M(C).
\end{equation}
This global perturbation statement is only the first layer. The sharper structural object is the exact regret geometry: competing couplings induce a convex polytope of cumulative deficits below the northwest-corner Fr\'echet extreme, and the regret is the support function of this polytope evaluated at the negative adjacent mixed differences of the cost. Replacing that polytope by its coordinatewise Fr\'echet box gives a projection-free, data-dependent bound. The relaxation reveals both where local violations matter and when several apparently harmful defects cannot be exploited simultaneously because of global marginal feasibility.

The matrix defect can in turn be traced back to geometry. If adjacent source and target increments lie near an $M$-acute cone, their negative Monge margins are bounded by the corresponding distances to the cone measured in the same Mahalanobis norm. The resulting sequence of implications is
\[
\boxed{\text{geometry}\ \Longrightarrow\ \text{Monge defect}\ \Longrightarrow\ \text{OT regret}.}
\]
This is the main bridge between exact theory and noisy high-dimensional data.

Three levels of structural information are kept separate throughout. An \emph{instance-level} condition asks only whether the observed cross-cost matrix is Monge. A \emph{chain-level} condition checks the actual adjacent increments of the ordered supports. A \emph{cone-level} condition such as (4) is stronger, but once verified it applies to every future pair of $K$-chains. This hierarchy clarifies the trade-off between conservatism and reuse: an instance-specific test is flexible but must be repeated, whereas a cone-level condition can be amortized across many transport queries.

The proposed viewpoint is complementary to other fast or structured OT constructions. Sliced Wasserstein (SW) obtains tractability by replacing a high-dimensional coupling with projected one-dimensional transports \cite{r7,r8}. Tree-Wasserstein and recent tree-sliced distances exploit tree metrics and tree systems \cite{r9,r12,r13}. Directional OT constrains admissible displacements and studies optimal couplings under that feasibility relation \cite{r34}. Learning elastic costs shapes Monge displacements by modifying the ground cost itself \cite{r30}. Our question is different from each of these: under a fixed high-dimensional ground cost, when does an order in the original space make the induced cost array Monge?

The distinction is also important for transport-based generative modeling. Dynamic optimal transport, flow matching, rectified flow, score-based diffusion, and Schr\"odinger bridges all use transport structure in different ways \cite{r14,r15,r16,r17,r18}, but a closed form for a static OT cost does not by itself imply an analytic generative flow. Such a conclusion would require an explicit transport map or velocity field. The results here concern the static geometry of pairwise transport costs; they are intended as a structural ingredient for future dynamic constructions rather than as a claim that diffusion, flow matching, or rectified-flow generators become closed form.

\subsection{Contributions}
This paper makes three contributions.
\begin{enumerate}[leftmargin=*]
\item \textbf{Cone geometry induces Monge structure.} We prove that $M$-acuteness of a cone is equivalent to the Monge property of every Mahalanobis cross-cost matrix generated by two $K$-chains. The result separates instance-, chain-, and uniform cone-level conditions, reduces to an entrywise test on $A^\top M A$ for finitely generated cones, and extends to smooth translation costs. Once the matrix is Monge, we explicitly invoke classical northwest-corner theory rather than treating the greedy transport rule as new.
\item \textbf{Local defects determine global regret.} Beyond exact chains, we represent northwest-corner regret exactly as the support function of a cumulative-deficit polytope determined by the marginals. Relaxing this polytope to its Fr\'echet box yields a projection-free, marginal-weighted sum of violated adjacent Monge inequalities. We characterize precisely when this relaxation is exact, prove worst-case sharpness under prescribed local defect budgets, and show that the sharpness construction is realizable by genuine squared-Euclidean cross-costs. Mahalanobis distance-to-cone deviations then yield a direct geometry-to-regret bound and a strict-margin radius of exact stability.
\item \textbf{Recognition, reuse, and limits of ordered structure.} We distinguish matrix-level recognition from reusable ground-space implications, prove that a scalar ranking alone is insufficient in dimensions $d\ge2$, define a regret-aligned score for comparing train-derived candidate orders, and quantify repeated-query savings once geometry and order are fixed. Continuous chain-supported measures and block-Monge chain covers are treated as extensions; unrestricted finite-width OT and global order discovery remain outside the claimed scope.
\end{enumerate}

\paragraph{Scope and boundaries.}
The contribution lies one level before classical Monge transportation and one level after arbitrary order discovery. We do not propose the northwest-corner rule, we do not claim a polynomial-time solution of global seriation, and we do not claim universal superiority over entropic, sliced, or tree-based OT. Instead, we identify when a high-dimensional order makes the original ground-cost problem structurally exact, and we quantify the loss when that structure is only approximate. Orders may be externally supplied (time, severity, known morphology) or estimated on training data and validated on a held-out set.

\section{Related Work}
\subsection{Computational OT and Surrogate Geometries}
Entropic regularization and Sinkhorn scaling are the standard numerical route to large-scale OT \cite{r3,r4,r5}. They preserve the original support geometry but solve a regularized problem iteratively. Recent IO-aware GPU solvers such as FlashSinkhorn substantially reduce memory traffic for entropic OT, strengthening the computational baseline without changing the distinction between regularized numerical transport and exact structural optimality \cite{r6}. Sliced Wasserstein takes a different route: measures are projected onto lines, one-dimensional Wasserstein distances are computed after sorting, and the resulting discrepancies are integrated or averaged over projection directions \cite{r7,r8}. This construction is statistically and computationally attractive, but its tractability comes from solving projected problems rather than the original high-dimensional coupling.

Tree-based distances exploit another source of structure. Tree-Wasserstein admits edge-additive formulas on tree metrics, and tree-sliced methods average such distances over sampled tree geometries \cite{r9}. Recent work has substantially strengthened this line: Tree-Sliced Wasserstein on systems of lines replaces one-dimensional slices by richer tree systems while retaining closed-form transport on the induced tree geometry, nonlinear projectional variants enlarge the class of admissible slices, and distance-based variants improve invariance and sampling efficiency \cite{r11,r12,r13}. These methods are therefore strong baselines whenever a tree surrogate is acceptable. Our analysis keeps the original Euclidean or Mahalanobis cost and asks when that very cost becomes Monge after ordering the supports.

Analytic OT also arises from special distributional families. Gaussian $W_2$ has the Bures--Wasserstein formula because the optimal map is affine \cite{r19,r20}. Radial and other structured families reduce the transport map to a lower-dimensional object. These cases reinforce the broader point that analytic transport requires structure; the present work isolates a different source of structure, namely compatibility between an order and the ground cost.

\subsection{Monge Matrices, Recognition, and Structured Costs}
Monge arrays are a classical tractable class in combinatorial optimization. Hoffman and Markowitz gave early greedy-solvability results for assignment and transportation problems \cite{r21}; Bein et al. characterized and generalized the Monge property for transportation arrays \cite{r22}; and Burkard et al. surveyed the algorithmic consequences of Monge structure across optimization problems \cite{r23}. Recognition is itself a mature topic. Rudolf gave polynomial algorithms for recognizing whether multi-dimensional arrays can be permuted into Monge form \cite{r24}. Recent seriation theory studies statistically optimal recovery of latent permutations from noisy Robinson-type similarities, including minimax-optimal polynomial-time recovery and active-query variants \cite{r26,r27}. This literature acts on an observed similarity or cost array and, when needed, searches for permutations that reveal latent order. Our transport criterion is different: it evaluates a rectangular source--target cost through the regret that a candidate order would induce.

The geometric side is less developed. H\"utter et al. showed that squared Euclidean distance matrices generated by ordered point configurations can be Monge when consecutive differences form acute angles, and they connected pre-Monge structure to seriation \cite{r25}. This observation is an important antecedent of our work. The present result differs in four ways: the source and target sequences are independent; the induced matrix may be rectangular; the cost is Mahalanobis rather than only Euclidean; and the condition is stated uniformly over an entire cone, which makes it reusable across all future pairs of chains. We also develop quantitative transport bounds once the Monge inequalities fail.

Nearly Monge arrays have been studied under other perturbation models, for example by allowing forbidden or incompatible entries in otherwise structured assignment arrays \cite{r28}. Our use of ``approximate Monge'' is different: all entries remain finite and we measure the degree to which the quadrangle inequalities themselves are violated. Nickel et al. recently revisited transportation problems under Monge costs and derived compact dual formulas from the northwest-corner structure \cite{r29}; their work further emphasizes that the transport problem is classical once Monge structure is given.

A second nearby line learns or designs the cost. Klein et al. learn elastic transport costs that encourage Monge displacements to follow low-dimensional engineered structure \cite{r30}. Our setup keeps $M$ fixed in the main theory and instead asks when the data order makes the resulting cross-cost array Monge. These two viewpoints are complementary: one shapes displacements by altering the cost, while the other diagnoses tractable order structure under a prescribed cost.

The cumulative-defect analysis also touches classical dependence theory. Fr\'echet--Hoeffding bounds identify the extremal cumulative mass compatible with fixed marginals, while comonotone and countermonotone couplings are classical extremizers for supermodular and submodular objectives \cite{r31,r32,r33}. We use these facts in a different role: the local mixed differences of an arbitrary cost matrix are allowed to have both signs, and the Fr\'echet width of each cut quantifies the largest amount by which a competing coupling can exploit a local violation. This yields a transport-regret bound and, below, an exact worst-case characterization under prescribed local defect budgets.

\begin{figure*}[t]
\centering
\includegraphics[width=0.95\textwidth]{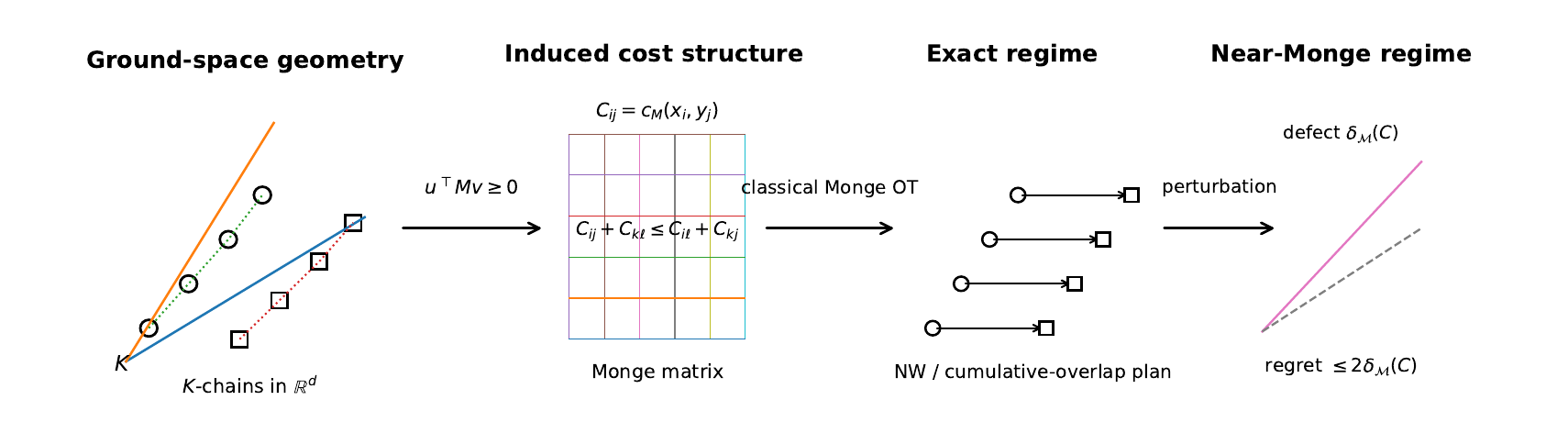}
\caption{From ground-space geometry to structured transport. The geometric implication from an $M$-acute cone to Monge cross-cost matrices is the new step studied here. Once a matrix is Monge, northwest-corner optimality is classical. When the exact inequalities are perturbed, local Monge defects and distance-to-cone quantities yield quantitative regret guarantees.}
\label{fig:pipeline}
\end{figure*}

\subsection{Ordered and Constrained Transport}
Order enters OT in several forms that should not be conflated. Directional OT constrains feasible displacements and constructs optimal directional couplings for suitable costs \cite{r34}. Convex displacement-constrained OT similarly restricts admissible pairs through conditions such as $y-x\in C$ \cite{r35,r36}. These are feasibility constraints on transported pairs. By contrast, our cone order acts on the relative order within the source and within the target; an optimal cone-chain plan need not satisfy $y-x\in K$ on each transported pair.

Order-preserving Wasserstein distances have been proposed for sequence matching, where the order is temporal \cite{r37}, while other methods impose inequalities directly on entries of the transport matrix \cite{r38}. Stochastic-order projections and causal or adapted OT encode measure-level dominance or information-flow constraints \cite{r39,r40,r41,r42,r43}. Vector quantile regression extends scalar quantiles through cyclically monotone convex-gradient transport maps \cite{r44}. Our object is narrower: we ask when an order on the ground space transforms the algebraic structure of the source--target cost array itself.

\section{Problem Setup and Structural Notions}
\subsection{Discrete Optimal Transport}
Let
\begin{equation}
\mu=\sum_{i=1}^{m}a_i\delta_{x_i},\qquad \nu=\sum_{j=1}^{n}b_j\delta_{y_j},
\end{equation}
with $a_i,b_j>0$ and $\sum_i a_i=\sum_j b_j=1$. Given a cost matrix $C\in\R^{m\times n}$,
\begin{equation}
\OT_C(a,b)=\min_{\pi\ge 0}\bigl\{\langle C,\pi\rangle:\ \pi\one_n=a,\ \pi^\top\one_m=b\bigr\}.
\end{equation}
For points in $\R^d$, the main theory uses
\begin{equation}
c_M(x,y)=\|x-y\|_M^2=(x-y)^\top M(x-y),\qquad M\succeq 0.
\end{equation}
Statements that require a genuine norm assume $M\succ0$.

\subsection{Cone Orders and Ordered Supports}
\begin{definition}[Cone order]
A closed convex pointed cone $K\subset\R^d$ induces
\begin{equation}
x\preceq_K y \quad\Longleftrightarrow\quad y-x\in K.
\end{equation}
A finite sequence $(x_1,\ldots,x_m)$ is a $K$-chain if $x_i\preceq_K x_k$ whenever $i<k$.
\end{definition}
The cone is called $M$-acute when
\begin{equation}
u^\top Mv\ge0,\qquad \forall u,v\in K.
\end{equation}
For $M=I$, this is the usual acuteness condition. If $K=A\R_+^r$ is finitely generated, the columns of $A\in\R^{d\times r}$ are its generators.

\subsection{Monge Structure and Local Margins}
\begin{definition}[Monge matrix]
A matrix $C\in\R^{m\times n}$ is Monge if
\begin{equation}
C_{ij}+C_{k\ell}\le C_{i\ell}+C_{kj},\qquad i<k,\ j<\ell.
\end{equation}
Its adjacent Monge margins are
\begin{equation}
\gamma_{ij}(C)=C_{i,j+1}+C_{i+1,j}-C_{ij}-C_{i+1,j+1}.
\end{equation}
\end{definition}
The full family of quadrangle inequalities is equivalent to $\gamma_{ij}(C)\ge0$ for every adjacent rectangle because every larger rectangle cross-difference is the sum of adjacent margins \cite{r23}.

\begin{definition}[Approximate Monge structure]
Let $\Mcal_{m,n}$ be the cone of Monge matrices. Define
\begin{equation}
\delta_M(C)=\inf_{\widetilde C\in\Mcal_{m,n}}\|C-\widetilde C\|_\infty,
\end{equation}
and
\begin{equation}
\varepsilon_{\mathrm{adj}}(C)=\max_{i<m,j<n}\pos{-\gamma_{ij}(C)}.
\end{equation}
\end{definition}
The first quantity is a global distance to the structured class; the second is immediately computable once the ordered cost matrix is available. Later we avoid even full matrix projection by expressing regret directly in terms of the local margins.

\subsection{Three Levels of Structural Information}
The same transport instance can be described at three levels.
\begin{enumerate}[leftmargin=*]
\item \emph{Instance-level}: the observed matrix $C$ is Monge, or approximately so.
\item \emph{Chain-level}: the actual adjacent increments $\Delta x_i=x_{i+1}-x_i$ and $\Delta y_j=y_{j+1}-y_j$ satisfy the required pairwise inner-product signs.
\item \emph{Cone-level}: every pair $u,v\in K$ satisfies (11), which guarantees the same property for all future chains contained in $K$.
\end{enumerate}
These levels are logically distinct. A fixed instance can be Monge even when no useful cone-level condition is available. Conversely, a uniform cone condition may be conservative for one data set but becomes valuable when the same geometry is reused across many transport queries.

\subsection{Order Dimension Versus Ambient Geometry}
A chain has a scalar order parameter, but this does not force its support to lie on a one-dimensional affine subspace. The distinction is essential: the monotone rule chooses which points are paired, whereas the cost still evaluates the full high-dimensional displacement.

\begin{proposition}[High-rank cone chains]
Let $K=\R_+^d$. For every $r\le \min\{d,m-1\}$ there exists a $K$-chain $x_1\preceq_K\cdots\preceq_K x_m$ such that
\begin{equation}
\rank[x_2-x_1,\ldots,x_m-x_{m-1}]\ge r.
\end{equation}
The same statement holds for the target chain. Consequently, exact Monge transport can occur on ordered supports whose local displacement directions span an arbitrarily high-dimensional subspace of the ambient space.
\end{proposition}
\emph{Construction.} Take the first $r$ increments to be distinct positive coordinate directions and the remaining increments to be any nonnegative vectors. Every increment belongs to $K$, while the first $r$ are linearly independent. This simple observation rules out an interpretation of the exact theorem as requiring a straight one-dimensional path in $\R^d$.

The primary high-dimensional experiment mirrors this construction with real images: successive positive perturbations are applied to distinct pixel patches, so the chain order is exact while the increment matrix has high numerical rank.

\section{Cone Geometry and Monge Structure}
\subsection{A Cross-Difference Identity}
Let $x_i\preceq_K x_k$ and $y_j\preceq_K y_\ell$. For Mahalanobis cost,
\begin{align}
&c_M(x_i,y_\ell)+c_M(x_k,y_j)-c_M(x_i,y_j)-c_M(x_k,y_\ell)\nonumber\\
&\qquad=2(x_k-x_i)^\top M(y_\ell-y_j).
\end{align}
This identity is the only algebra needed to connect the ground geometry to Monge matrices. Its strength comes from quantification over all ordered pairs rather than from the complexity of the expansion itself.

\subsection{Cone--Monge Equivalence}
\begin{theorem}[Cone--Monge equivalence]
Let $K\subset\R^d$ be a closed convex cone and $M\succeq0$. The following are equivalent.
\begin{enumerate}[label=(\roman*),leftmargin=*]
\item $u^\top Mv\ge0$ for every $u,v\in K$.
\item For every pair of finite $K$-chains $x_1\preceq_K\cdots\preceq_K x_m$ and $y_1\preceq_K\cdots\preceq_K y_n$, the cross-cost matrix $C_{ij}=c_M(x_i,y_j)$ is Monge.
\end{enumerate}
Under either condition, the northwest-corner plan is optimal for every choice of positive chain marginals $a,b$ by the classical Monge transportation theorem.
\end{theorem}
\emph{Proof sketch.} If (i) holds, every increment $x_k-x_i$ and $y_\ell-y_j$ belongs to $K$, so (17) is nonnegative and hence the Monge inequality holds. Conversely, take arbitrary $u,v\in K$ and two-point chains $x_1=0,x_2=u$ and $y_1=0,y_2=v$. The Monge inequality for their $2\times2$ cross-cost matrix yields $u^\top Mv\ge0$. The final northwest-corner statement is classical \cite{r21,r22}. Full details are given in Appendix~\ref{app:equivalence}.

Theorem 1 should be read as a geometric characterization of a family of Monge matrices. It does not replace matrix-level recognition. Instead, it identifies a condition in the original feature space that implies the matrix property for all chain pairs simultaneously.

\subsection{Chain-Level and Uniform Conditions}
For actual chains, Theorem 1 can be weakened.
\begin{proposition}[Hierarchy of sufficient conditions]
For Mahalanobis cost, the following implications hold:
\begin{equation}
\text{cone acuteness}\Longrightarrow\text{adjacent-chain condition}\Longrightarrow\text{instance-level Monge structure},
\end{equation}
where the adjacent chain condition is
\begin{equation}
(\Delta x_i)^\top M(\Delta y_j)\ge0,\qquad \forall i<m,\ j<n.
\end{equation}
Moreover, (19) is equivalent to nonnegativity of all adjacent margins of the induced cost matrix.
\end{proposition}
This hierarchy is useful empirically. The uniform cone condition is appropriate when a reusable geometry is known. The adjacent condition is less conservative when the order has already been estimated from a specific data set. The cost-matrix condition is the most permissive but requires the matrix itself.

\subsection{Finite Generators and Smooth Translation Costs}
\begin{proposition}[Finite-generator characterization]
Let $K=A\R_+^r$. Then $K$ is $M$-acute if and only if $A^\top M A$ is entrywise nonnegative.
\end{proposition}
The condition replaces an infinite family of inequalities over $K$ by a finite matrix test. When $M$ and $A$ are fixed, the test is performed once and reused for all subsequent chain pairs.

The same order--cost principle is not confined to quadratic costs.
\begin{proposition}[Smooth translation costs]
Let $c(x,y)=h(x-y)$ with $h\in C^2(\R^d)$ convex. If
\begin{equation}
u^\top \nabla^2 h(z)v\ge0,\qquad \forall z\in\R^d,\ u,v\in K,
\end{equation}
then every two $K$-chains induce a Monge cross-cost matrix.
\end{proposition}
For $h(z)=z^\top Mz$, Proposition 4 reduces to Theorem 1. The quadratic case remains the central one because it admits the sharp cone-level equivalence and a natural Mahalanobis perturbation analysis.

\subsection{Mahalanobis Coordinates}
The geometry can be expressed without ambiguity under a linear change of coordinates.
\begin{proposition}[Euclidean representation of Mahalanobis acuteness]
Let $M=L^\top L$ with $L$ invertible. Then
\begin{equation}
u^\top Mv=(Lu)^\top(Lv),
\end{equation}
so $K$ is $M$-acute if and only if $LK$ is acute in the Euclidean inner product. Moreover,
\begin{equation}
c_M(x,y)=\|Lx-Ly\|_2^2.
\end{equation}
\end{proposition}
Thus the cone--Monge theorem is invariant under an invertible linear reparameterization when the cone and the metric are transformed together.

This observation is useful for interpretation. A non-diagonal Mahalanobis matrix does not introduce a different notion of order--cost compatibility; it changes the coordinate system in which acuteness is judged. It also explains why the perturbation theorem below is most naturally written in the $M$-norm rather than by inserting an external operator-norm constant.

\section{Exact and Approximate Monge Transport}
\subsection{Exact Chain Transport}
Let $A_i=\sum_{r\le i}a_r$ and $B_j=\sum_{s\le j}b_s$, with $A_0=B_0=0$. The northwest-corner plan has entries
\begin{equation}
\omega_{ij}=\pos{\min(A_i,B_j)-\max(A_{i-1},B_{j-1})}.
\end{equation}
\begin{corollary}[Exact transport on cone chains]
Under Theorem 1,
\begin{equation}
\OT_C(a,b)=\sum_{i,j}\omega_{ij}c_M(x_i,y_j).
\end{equation}
For equal weights and $m=n$,
\begin{equation}
\OT_C(a,b)=\frac1n\sum_{i=1}^n\|x_i-y_i\|_M^2.
\end{equation}
\end{corollary}
The corollary is a geometric route to a classical Monge transportation solution, not a new northwest-corner algorithm.

\subsection{Distance-to-Monge Regret}
The first perturbation statement is independent of cones.
\begin{theorem}[Distance-to-Monge regret]
Let $C\in\R^{m\times n}$ and let $\widetilde C\in\Mcal_{m,n}$ satisfy $\|C-\widetilde C\|_\infty\le\delta$. For the same probability marginals $a,b$, let $\pi_{\NW}$ be the northwest-corner plan. Then
\begin{equation}
0\le\langle C,\pi_{\NW}\rangle-\OT_C(a,b)\le2\delta.
\end{equation}
Consequently,
\begin{equation}
\langle C,\pi_{\NW}\rangle-\OT_C(a,b)\le2\delta_M(C).
\end{equation}
\end{theorem}
The proof uses only two facts: $\pi_{\NW}$ is optimal for every Monge matrix with the same marginals, and every probability transport plan has unit total mass. The bound therefore isolates the effect of cost perturbation from any geometric model that produced it.

\subsection{Exact Regret Geometry}
Computing $\delta_M(C)$ is useful for a global perturbation statement, but it hides how different local violations interact through the marginal constraints. The interaction can be described exactly in cumulative coordinates. For a coupling $\pi$, let
\begin{equation}
H_\pi(i,j)=\sum_{r\le i,\,s\le j}\pi_{rs},\qquad i<m,\ j<n,
\end{equation}
and let $H_{\NW}(i,j)=\min(A_i,B_j)$. Define the cumulative deficit
\begin{equation}
D^\pi_{ij}=H_{\NW}(i,j)-H_\pi(i,j),
\end{equation}
and the feasible deficit polytope
\begin{equation}
\Dcal(a,b)=\{D^\pi:\pi\in\Pi(a,b)\}\subset\R^{(m-1)(n-1)}.
\end{equation}
Since $H_{\NW}$ is the Fr\'echet--Hoeffding upper bound, $D^\pi_{ij}\ge0$ for every feasible coupling.

\begin{theorem}[Exact regret representation]
For every cost matrix $C$ and probability marginals $a,b$,
\begin{equation}
\langle C,\pi_{\NW}\rangle-\OT_C(a,b)=\max_{D\in\Dcal(a,b)}\left\{-\sum_{i=1}^{m-1}\sum_{j=1}^{n-1}\gamma_{ij}(C)D_{ij}\right\}.
\end{equation}
Equivalently, northwest-corner regret is the support function of $\Dcal(a,b)$ evaluated at $-\gamma(C)$.
\end{theorem}
The theorem separates cost information from marginal geometry. The cost enters only through its adjacent mixed differences; the admissible interactions among those defects are encoded entirely by $\Dcal(a,b)$.

Every coupling also satisfies the Fr\'echet--Hoeffding lower bound $\max\{0,A_i+B_j-1\}\le H_\pi(i,j)$, hence
\begin{equation}
0\le D^\pi_{ij}\le w_{ij}:=\min(A_i,B_j)-\max\{0,A_i+B_j-1\}.
\end{equation}
Thus $\Dcal(a,b)$ lies inside the axis-aligned Fr\'echet box
\begin{equation}
\Bcal(a,b)=\prod_{i<m,j<n}[0,w_{ij}].
\end{equation}
Replacing the true deficit polytope by this box gives a closed-form relaxation.

\subsection{A Projection-Free Cumulative-Defect Bound}
\begin{theorem}[Cumulative-defect regret]
For every cost matrix $C$ and probability marginals $a,b$,
\begin{equation}
0\le\langle C,\pi_{\NW}\rangle-\OT_C(a,b)\le\sum_{i=1}^{m-1}\sum_{j=1}^{n-1}w_{ij}\pos{-\gamma_{ij}(C)}.
\end{equation}
\end{theorem}
The right-hand side is the support function of $\Bcal(a,b)$ at $-\gamma(C)$, whereas Theorem 3 uses the smaller polytope $\Dcal(a,b)$. Their difference therefore measures the interaction among local defects imposed by global transport feasibility. The bound is not claimed to dominate Theorem 2; the two use different information.

\begin{corollary}[When the cumulative bound is exact]
Equality holds in Theorem 4 if and only if there exists a feasible coupling $\pi$ such that
\begin{equation}
D^\pi_{ij}=w_{ij}\ \text{whenever }\gamma_{ij}(C)<0,\qquad D^\pi_{ij}=0\ \text{whenever }\gamma_{ij}(C)>0.
\end{equation}
No condition is imposed on cells with zero margin. In particular, every anti-Monge cost attains equality through the countermonotone Fr\'echet coupling.
\end{corollary}
Thus the relaxation is exact whenever the sign pattern of the mixed differences admits a feasible deficit field that reaches the coordinatewise optimal face of the Fr\'echet box. When simultaneous saturation is impossible, the gap between the box relaxation and the exact regret quantifies defect interaction.

The weights in (34) are not an artifact of the proof. They are the smallest possible universal coefficients when only lower bounds on adjacent margins are available.

\begin{theorem}[Worst-case sharpness under local defect budgets]
Fix probability marginals $a,b$ and nonnegative numbers $d_{ij}$ for $i<m,j<n$. Let
\begin{equation}
\mathcal{C}(d)=\{C\in\R^{m\times n}:\gamma_{ij}(C)\ge-d_{ij}\ \text{for all }i,j\}.
\end{equation}
Then
\begin{equation}
\sup_{C\in\mathcal{C}(d)}\{\langle C,\pi_{\NW}\rangle-\OT_C(a,b)\}=\sum_{i=1}^{m-1}\sum_{j=1}^{n-1}w_{ij}d_{ij}.
\end{equation}
In particular, every coefficient $w_{ij}$ in Theorem 4 is individually unimprovable without information beyond the local defect budgets.
\end{theorem}
The upper bound is Theorem 4. Equality is attained by an anti-Monge cost with adjacent margins exactly $-d_{ij}$; the countermonotone Fr\'echet coupling attains the lower cumulative bound at every cut. Hence (34) is the exact worst-case regret determined by the local mixed-difference information alone. This is a worst-case statement, not a claim of pointwise equality for every mixed-sign defect pattern.

\begin{corollary}[Euclidean realizability of prescribed defect patterns]
Let $D=(d_{ij})\in\R_+^{(m-1)\times(n-1)}$ and $r=\rank(D)$. There exist source and target sequences in $\R^r$ whose squared-Euclidean cross-cost matrix satisfies $\gamma_{ij}=-d_{ij}$ for every adjacent rectangle. Consequently, the worst-case equality in Theorem 5 can be attained within genuine squared-Euclidean transport costs, not only by abstract cost arrays.
\end{corollary}
The construction follows from a rank factorization $-D/2=UV^\top$: use the rows of $U$ and $V$ as source and target increments and take cumulative sums. This result does not assert that the worst-case pattern lies near a common acute cone; it only shows that the sharpness mechanism is geometrically realizable within high-dimensional Euclidean OT.

The same argument gives a coarse uniform corollary.
\begin{proposition}[Uniform adjacent-violation bound]
If $\gamma_{ij}(C)\ge-\varepsilon$ for all adjacent indices, then
\begin{equation}
\langle C,\pi_{\NW}\rangle-\OT_C(a,b)\le\varepsilon\sum_{i,j}w_{ij}\le\frac{\varepsilon}{2}(m-1)(n-1).
\end{equation}
Moreover,
\begin{equation}
\delta_M(C)\le\frac{\varepsilon}{4}(m-1)(n-1).
\end{equation}
\end{proposition}
The latter estimate is a worst-case structural bound; it is not intended to be tight in large matrices.

\subsection{Relation to $c$-Cyclical Monotonicity}
Monge matrices and cyclical monotonicity describe optimality at different levels. Cyclical monotonicity is the general support condition for optimal transport, while the Monge quadrangle inequality is a stronger ordered structure that makes a particular support optimal without solving the global problem. On totally ordered source and target supports, the connection is especially simple.

\begin{proposition}[Ordered Monge support is cyclically monotone]
Let $C$ be a Monge matrix and let $\pi_{\NW}$ be the northwest-corner plan. Its support is $c$-cyclically monotone. In particular, every finite cyclic reassignment of mass among support pairs has cost no smaller than the ordered assignment.
\end{proposition}
The result follows by decomposing any permutation of a totally ordered list into inversions and removing those inversions one at a time with the Monge exchange inequality. This proposition does not replace the general theory of cyclic monotonicity; it explains why a pairwise quadrangle inequality becomes sufficient once the supports are ordered.

\subsection{From Cone Perturbations to Monge Defects}
The matrix bounds become geometrically meaningful once the margins are expressed through adjacent increments. Assume now that $M\succ0$ and let
\begin{equation}
\eta^x_{i,M}=\dist_M(\Delta x_i,K):=\inf_{u\in K}\|\Delta x_i-u\|_M,
\end{equation}
with $\eta^y_{j,M}$ defined analogously. Let
\begin{equation}
r^x_{i,M}=\|\Delta x_i\|_M,\qquad r^y_{j,M}=\|\Delta y_j\|_M.
\end{equation}

\begin{theorem}[Mahalanobis geometric perturbation bound]
Assume $K$ is $M$-acute. Then
\begin{equation}
\pos{-\gamma_{ij}(C)}\le2\bigl(\eta^x_{i,M}r^y_{j,M}+\eta^y_{j,M}r^x_{i,M}+\eta^x_{i,M}\eta^y_{j,M}\bigr).
\end{equation}
Consequently,
\begin{align}
\langle C,\pi_{\NW}\rangle-\OT_C(a,b)
&\le2\sum_{i=1}^{m-1}\sum_{j=1}^{n-1}w_{ij}\bigl(\eta^x_{i,M}r^y_{j,M}\nonumber\\
&\hspace{16mm}+\eta^y_{j,M}r^x_{i,M}+\eta^x_{i,M}\eta^y_{j,M}\bigr).
\end{align}
\end{theorem}
Unlike a Euclidean projection argument followed by an operator-norm conversion, Theorem 6 measures both the defect and the increment size in the geometry of the transport cost itself. For $M=I$ it reduces to the Euclidean statement. The result is useful even when the cone condition is only approximate: it converts a directly interpretable geometric deviation into an upper bound on the loss of using the monotone plan.

\begin{remark}[Oscillation-capped geometric certificate]
Let $B_{\mathrm{geom}}$ denote the right-hand side of (43) and define
\begin{equation}
\osc(C)=\max_{i,j}C_{ij}-\min_{i,j}C_{ij}.
\end{equation}
Because both $\pi_{\NW}$ and an optimal coupling have unit mass,
\begin{equation}
0\le\langle C,\pi_{\NW}\rangle-\OT_C(a,b)\le\min\{B_{\mathrm{geom}},\osc(C)\}.
\end{equation}
Thus $B_{\mathrm{geom}}/\osc(C)<1$ certifies an improvement over the trivial cost-range bound. When this ratio is large, the geometric estimate remains valid but should be interpreted as a conservative stability certificate rather than a quantitatively tight approximation. No first-order tightness claim is made for generic perturbations.
\end{remark}

A strict angular margin also gives a finite robustness radius for exactness.
\begin{corollary}[Strict-margin stability]
Let reference increments $u_i,v_j\in K\setminus\{0\}$ satisfy
\begin{equation}
\frac{u_i^\top Mv_j}{\|u_i\|_M\|v_j\|_M}\ge\kappa>0\qquad\text{for all }i,j.
\end{equation}
Suppose the observed increments are $\Delta x_i=u_i+e_i$ and $\Delta y_j=v_j+f_j$ with
\begin{equation}
\|e_i\|_M\le\rho\|u_i\|_M,\qquad \|f_j\|_M\le\rho\|v_j\|_M.
\end{equation}
If $2\rho+\rho^2\le\kappa$, then every adjacent Monge margin of the perturbed cross-cost matrix is nonnegative. Hence the northwest-corner plan remains exactly optimal. Equivalently, exactness is guaranteed for
\begin{equation}
\rho\le\sqrt{1+\kappa}-1.
\end{equation}
\end{corollary}
This criterion separates two regimes that are blurred by a generic perturbation bound: below the margin radius the exact Monge structure provably survives; beyond it, Theorems 4 and 6 quantify the loss rather than asserting exactness.

\subsection{Interpreting the Bounds}
Theorems 2--6 answer different questions. The distance-to-Monge bound is the cleanest perturbation statement and can be evaluated exactly for small matrices through linear programming. The cumulative-defect bound avoids this projection and is sensitive to both defect location and marginals. The geometric bound does not require forming a Monge projection at all, but it can be conservative because local distance-to-cone information discards cancellation between increments and accumulates triangle/Cauchy--Schwarz losses. For large deviations, the estimate should therefore be read as a qualitative stability certificate unless it improves materially on the cost-range cap in Remark 1.

The experiments report $B_{\mathrm{geom}}/R_{\mathrm{obs}}$ when $R_{\mathrm{obs}}>0$, the dimensionless informativeness ratio $B_{\mathrm{geom}}/\osc(C)$, and the corresponding cumulative-defect quantities. A bound that is formally correct but numerically vacuous is treated as such rather than presented as empirical confirmation.

\section{Continuous Chains and Structured Extensions}
\subsection{Continuous Chain-Supported Measures}
The discrete formula is not tied to empirical measures. Let $I_X,I_Y\subset\R$ be intervals, let $\gamma_X:I_X\to\R^d$ and $\gamma_Y:I_Y\to\R^d$ be Borel measurable injective $K$-monotone maps, and let $\alpha,\beta$ be probability measures on the parameter intervals. Write
\begin{equation}
\mu=(\gamma_X)_\#\alpha,\qquad \nu=(\gamma_Y)_\#\beta.
\end{equation}
\begin{theorem}[Continuous ordered transport]
Assume $c$ is lower semicontinuous, bounded below, integrable under the couplings considered, and Monge-compatible with $K$. Then the scalar cost
\begin{equation}
\widetilde c(s,t)=c(\gamma_X(s),\gamma_Y(t))
\end{equation}
has the one-dimensional Monge property. A comonotone coupling of $\alpha$ and $\beta$, formed from a common $U\sim\mathrm{Unif}(0,1)$ and generalized quantiles $F_\alpha^{-1}(U),F_\beta^{-1}(U)$, pushes forward to an optimal coupling of $\mu$ and $\nu$.
\end{theorem}
The theorem makes the role of dimension precise. The ordering coordinate is one-dimensional, but the transported points and the ground cost remain high-dimensional. This is fundamentally different from projecting the data themselves to one-dimensional slices: no component of the high-dimensional displacement is removed from the cost.

\subsection{Chain Covers and Block-Monge Structure}
A general partial order may have width larger than one. Let
\begin{equation}
X=\bigcup_{a=1}^{w_X}X^{(a)},\qquad Y=\bigcup_{b=1}^{w_Y}Y^{(b)}
\end{equation}
be chain covers. By Dilworth's theorem, the minimum number of chains equals the width of the finite poset \cite{r45}.

\begin{proposition}[Block-Monge structure]
If $K$ is $M$-acute, every source-chain/target-chain cross-cost block $C^{ab}$ is Monge.
\end{proposition}
This does not make unrestricted finite-width OT closed form. It identifies a block-structured matrix in which the within-block problems are analytic and the remaining difficulty is the allocation of mass across chains.

One transparent restricted model is useful for separating these two difficulties. Suppose
\begin{equation}
\mu=\sum_{a=1}^{w_X}\alpha_a\mu_a,\qquad \nu=\sum_{b=1}^{w_Y}\beta_b\nu_b,
\end{equation}
where $\mu_a$ and $\nu_b$ are chain-supported probability measures. Restrict attention to plans
\begin{equation}
\PiMix=\left\{\sum_{a,b}q_{ab}\pi_{ab}:\ q\in\Pi(\alpha,\beta),\ \pi_{ab}\in\Pi(\mu_a,\nu_b)\right\}.
\end{equation}
Then
\begin{equation}
\inf_{\pi\in\PiMix}\int c\,d\pi=\min_{q\in\Pi(\alpha,\beta)}\sum_{a,b}q_{ab}\OT_c(\mu_a,\nu_b).
\end{equation}
Each inner cost is analytic under Theorem 1; only the outer chain-mass allocation remains numerical. We present this as a restricted decomposition, not as a general finite-width solution.

\section{Order Recognition and Computational Consequences}
\subsection{Matrix Recognition Versus Geometric Implication}
There are two conceptually different ways to exploit Monge structure. The classical route forms $C$ and asks whether it is Monge, or whether row and column permutations can make it Monge \cite{r23,r24,r25}. This route is instance-specific and can adapt to structures that are not induced by a cone. The geometric route studied here fixes an order model first and asks whether the ground cost makes every compatible chain pair Monge. Neither dominates the other.

The distinction becomes important when the same geometry is reused. Suppose $K=A\R_+^r$ and $M$ are fixed across $Q$ transport queries. The compatibility matrix $A^\top M A$ is checked once. For query $q$ with chain lengths $m_q,n_q$, the northwest-corner plan touches at most $m_q+n_q-1$ source--target pairs. In contrast, forming the complete cross-cost matrix requires $m_qn_q$ pairwise costs for each query. The benefit is therefore not that northwest-corner transport is new, but that a single ground-space condition can justify sparse cost evaluation repeatedly.

\begin{proposition}[Amortized pair-evaluation complexity]
Assume the chain orders are given, $K=A\R_+^r$ is $M$-acute, and one pairwise Mahalanobis cost evaluation has complexity $\tau_c$. Across $Q$ chain-pair queries, exact transport requires at most
\begin{equation}
\sum_{q=1}^{Q}(m_q+n_q-1)
\end{equation}
pairwise cost evaluations after the one-time compatibility check. Constructing all cross-cost matrices requires $\sum_qm_qn_q$ evaluations. Order discovery and representation learning are separate costs and are not included in this comparison.
\end{proposition}
This proposition is intentionally narrow. It isolates the computational advantage that follows from structure alone while leaving order discovery visible rather than hiding it in preprocessing.

The comparison also clarifies when the geometric route should not be used. If there is only a single transport query and no credible order is available, direct matrix construction followed by a general OT solver or a matrix-level seriation method can be preferable. If the same representation, metric, and order model are reused across many pairs of distributions, a one-time geometric implication can avoid repeated full-matrix verification. The relevant quantity is therefore not merely the asymptotic cost of northwest-corner transport, which is classical, but the total cost of discovering, verifying, and reusing structure.

For $Q$ repeated queries with common feature dimension and comparable chain lengths $n$, full cost construction scales as $O(Qn^2\tau_c)$ pair evaluations, whereas ordered sparse evaluation scales as $O(Qn\tau_c)$ after the order and geometry are fixed. This comparison excludes the cost of learning the representation and order; those terms must be reported separately. In the final experiments we therefore decompose wall-clock time into representation preprocessing, order fitting, sorting, cost evaluation, matrix-level Monge testing when applicable, transport, and total runtime.

\subsection{How an Order May Be Obtained}
The main theorems do not require that the order be learned by a particular algorithm. In some applications it is supplied by the problem, as in time, severity, or a known morphometric factor. In others it can be estimated from training data by a scalar score, a seriation method, or a learned cone. The distinction matters experimentally because a favorable order should not be selected on the test set.

Our primary experiments therefore use three regimes. Synthetic tests use a known cone and known order. Morpho-MNIST uses a released morphometric variable as an externally observed order and, separately, compares simple train-derived ordering rules such as paired-difference directions, principal components, and ordinal linear scores. A longitudinal experiment, if it passes a prespecified validation gate, estimates its direction only from time-ordered training trajectories. This protocol tests the transport theory without making a separate contribution claim about order learning.

We therefore decouple verification of ordered transport from discovery of the order: the order is treated either as an external inductive bias (time, severity, known morphology) or as a train-set-derived hypothesis validated by $U$ on held-out data. The regret theory is conditional on any fixed candidate order and remains valid even when global seriation is unknown.

A learned cone is nevertheless possible. For $K_A=A\R_+^r$, one can minimize training distance-to-cone while enforcing entrywise nonnegativity of $A^\top M A$, generator normalization, pointedness, and nonvanishing empirical coverage. We treat this as an auxiliary extension and give the precise objective and a standard structural-risk bound in Appendix~\ref{app:learnedcone}. It is not required for any main theorem or primary experiment.

\subsection{A Regret-Aligned Criterion for Candidate Orders}
The exact regret representation also gives a principled way to compare candidate orderings without turning this paper into a seriation paper. For permutations $\sigma\in S_m$ and $\tau\in S_n$, let $C^{\sigma,\tau}$ denote the cross-cost matrix after reordering source and target supports. Define
\begin{equation}
U(\sigma,\tau)=\sum_{i<m,j<n}w_{ij}\pos{-\gamma_{ij}(C^{\sigma,\tau})}.
\end{equation}
By Theorem 4, $U(\sigma,\tau)$ upper-bounds the regret of using the northwest-corner plan under that candidate order. It is therefore a regret-aligned validation score rather than a generic similarity-based seriation loss. We do not claim a new polynomial-time global optimizer of (56); exact minimization over permutations is a separate combinatorial problem. In experiments, $U$ is used only to compare or locally refine a small set of train-derived candidate orders, with all selection frozen before held-out testing.

This distinction separates our use of ordering from recent seriation theory. Seriation recovers a latent permutation from noisy symmetric Robinson-type similarities \cite{r26,r27}. Equation (56), by contrast, is built from a rectangular source--target transport cost and evaluated by the transport regret that a candidate order can induce. A good candidate order for one objective need not optimize the other.

\subsection{Auxiliary Ordered Transport Objectives}
The exact results concern the original Kantorovich objective. Other tasks may require a metric on an ordered class or a directed feasibility relation. For completeness, Appendix~\ref{app:auxobj} records three auxiliary objects: a canonical chain metric built from a fixed order coordinate, a directed cone OT cost with $y-x\in K$, and a soft distance-to-cone penalized objective. These objects serve different purposes and should not be conflated with the exact Monge result. In particular, the soft and directed formulations are generally numerical and are not advertised as closed-form transport.

\section{Static Transport Structure and Generative Dynamics}
The present theory concerns static optimal transport costs and couplings. This is deliberately narrower than dynamic OT, flow matching, rectified flow, score-based diffusion, and Schr\"odinger bridges \cite{r14,r15,r16,r17,r18}. Flow-matching and rectified-flow methods learn Eulerian velocity fields from prescribed conditional paths or sample pairings. Our results can supply an exact or certified endpoint coupling on ordered supports, but that coupling alone does not identify a globally defined velocity field, score, or Markovian sampler. In particular, a discrete northwest-corner plan may split mass, and its displacement interpolation need not correspond to a single closed-form ODE field without additional injectivity, no-crossing, or regularity assumptions. Our contribution to generative modeling is therefore at the coupling layer: an ordered endpoint coupling can serve as supervision or as a structural screening rule inside a larger pipeline, while a large defect certificate signals that a general numerical coupling is still needed. We provide the pairing, not the flow.

\subsection{Displacement Interpolation on Exact Chains}
The distinction between a static coupling and a dynamic generative model does not prevent the exact chain result from inducing a closed-form Wasserstein interpolation. When $M\succ0$, let $W_{2,M}$ denote the quadratic Wasserstein distance associated with $c_M$. If $\pi_{\NW}$ is the optimal coupling from Corollary 1, define
\begin{equation}
\mu_t=((1-t)x+ty)_\#\pi_{\NW},\qquad t\in[0,1].
\end{equation}
For empirical chains this measure is explicit:
\begin{equation}
\mu_t=\sum_{i,j}\omega_{ij}\delta_{(1-t)x_i+ty_j},
\end{equation}
with at most $m+n-1$ active segments.

\begin{proposition}[Exact-chain displacement interpolation]
Under the assumptions of Corollary 1 and $M\succ0$, the path in (57) is a constant-speed $W_{2,M}$ geodesic: for all $0\le s<t\le1$,
\begin{equation}
W_{2,M}(\mu_s,\mu_t)=(t-s)W_{2,M}(\mu_0,\mu_1).
\end{equation}
If the optimal coupling is induced by a map, the interpolation is deterministic along that map; otherwise it remains plan-valued.
\end{proposition}
The proposition is a standard consequence of displacement interpolation along an optimal quadratic coupling, applied here to the analytically available chain coupling. It gives an explicit path of endpoint marginals, but it should not be confused with a closed-form diffusion, flow-matching, or rectified-flow dynamics: those models additionally require a Markovian velocity or score field defined beyond the finite coupling support.

The same caution applies to high dimensionality. A cone chain has a one-dimensional order parameter, even though its points and its Mahalanobis displacements remain in $\R^d$. The contribution is therefore not a universal high-dimensional closed form. It is a characterization of when high-dimensional geometry preserves enough order for the original cost to inherit the Monge property, together with quantitative bounds when that property is only approximate.

\section{Boundary Regimes and Extensions}
\subsection{Comparability in Unstructured High Dimension}
Cone order is useful only when the data carry a compatible directional structure. This limitation can be quantified independently of any transport algorithm. Let $U$ be uniformly distributed on the unit sphere and define the normalized solid angle of a cone by
\begin{equation}
\Omega(K)=\Pr(U\in K)=\frac{\area(K\cap S^{d-1})}{\area(S^{d-1})}.
\end{equation}
\begin{proposition}[Solid-angle limitation]
Suppose the direction of $Y-X$ is uniform on $S^{d-1}$ and independent of its norm. Then
\begin{equation}
\Pr(X\preceq_K Y)=\Omega(K).
\end{equation}
If $K$ is pointed, the two-way comparability probability is $2\Omega(K)$. In particular, for $K=\R_+^d$,
\begin{equation}
\Pr(X\preceq_K Y\ \text{or}\ Y\preceq_K X)=2^{1-d}.
\end{equation}
\end{proposition}
The orthant example shows why a useful order should not be expected in an isotropic high-dimensional cloud. A wider cone increases comparability but may weaken semantic specificity; an excessively narrow cone preserves a strong direction but covers little data. The appropriate regime is therefore a property of the representation and the task, not of OT alone. This is why the empirical study reports both structural coverage and Monge defect rather than presenting order as a universal preprocessing step.

The proposition also explains the role of learned representations. A raw pixel or clinical space may exhibit weak comparability even when a lower-dimensional or task-aligned representation admits a coherent progression direction. Such a representation is useful only if the order is learned on training data and remains stable on held-out observations; otherwise the ordering step is merely retrospective.

\subsection{$p$-Wasserstein Costs Beyond the Quadratic Case}
The sharp cone-level equivalence in Theorem 1 is specific to the quadratic Mahalanobis structure, but the order argument extends to other costs. One important class is separable.
\begin{proposition}[Separable convex costs]
Let
\begin{equation}
c(x,y)=\sum_{r=1}^{d}h_r(x_r-y_r),
\end{equation}
where every $h_r:\R\to\R$ is convex. Then the orthant order $K=\R_+^d$ is Monge-compatible with $c$. In particular, the result applies to $\ell_p^p$ costs $c(x,y)=\sum_r|x_r-y_r|^p$ for every $p\ge1$.
\end{proposition}
Each coordinate contributes a one-dimensional convex Monge cost, and sums preserve the quadrangle inequality. Thus the same cumulative-overlap plan is exact on orthant chains for separable $p$-costs. By contrast, the nonseparable cost $\|x-y\|_2^p$ with $p\ne2$ does not reduce to the bilinear cross-difference in (17); Proposition 4 supplies a sufficient cone-Hessian condition, but we do not claim a sharp characterization in that case. This is why the paper centers the theory on $p=2$ rather than overstating a general $p$ result.

\begin{table}[t]
\centering
\caption{What different forms of ordering information guarantee.}
\label{tab:orderinginfo}
\small
\begin{tabular}{p{0.27\columnwidth}p{0.65\columnwidth}}
\toprule
Information & Consequence\\
\midrule
Scalar score only & Produces a ranking; does not imply Monge cost in $d\ge2$.\\
Observed adjacent increments & Monge iff all source--target adjacent inner products have the correct sign.\\
$M$-acute cone & Guarantees Monge structure uniformly for every pair of $K$-chains.\\
Full cost matrix & Allows direct Monge/permuted-Monge recognition, but only after matrix construction.\\
\bottomrule
\end{tabular}
\end{table}

\subsection{Why an Arbitrary Scalar Order Is Not Enough}
The existence of a scalar ranking does not by itself make a high-dimensional cost matrix Monge. This point is important because any finite data set can be sorted by a score; the nontrivial requirement is that the resulting source and target increments interact with the ground cost with the correct sign.

\begin{proposition}[Scalar ordering does not imply Monge structure]
Let $d\ge2$ and let $s(x)=v^\top x$ be any nonzero linear score. There exist two source points $x_1,x_2$ and two target points $y_1,y_2$ such that
\begin{equation}
s(x_1)<s(x_2),\qquad s(y_1)<s(y_2),
\end{equation}
but the squared Euclidean cross-cost matrix violates the Monge inequality.
\end{proposition}
\emph{Construction.} After a rotation take $v=e_1$. Let $u=(1,a,0,\ldots,0)$ and $w=(1,-a,0,\ldots,0)$ with $a>1$, and set $x_2-x_1=u$, $y_2-y_1=w$. Both score increments are positive, yet $u^\top w=1-a^2<0$. By (17), the ordered $2\times2$ cost matrix is anti-Monge rather than Monge.

This counterexample clarifies the role of the cone. A scalar score supplies an order coordinate; it does not control the high-dimensional directions that are discarded by that coordinate. An acute cone constrains those directions jointly, which is precisely what makes the original high-dimensional cost obey the quadrangle inequality. In experiments, PCA, ordinal linear scores, and random scores are therefore meaningful baselines rather than equivalent formulations of the theory.

\subsection{Choosing Between Structured and General OT}
The theory suggests a simple model-selection principle. If an externally meaningful order yields low adjacent Monge violation and the cumulative-defect bound is small, monotone transport has a quantitative justification. If the order is weak, unstable across splits, or produces a vacuous regret bound, the correct response is not to enlarge the cone until the experiment succeeds; it is to use a general solver or a different surrogate geometry. SW remains attractive when scalable symmetric discrepancy is the main objective, tree-based distances when a hierarchy is meaningful, and classical or entropic OT when no reliable ordered structure is present.

This boundary is part of the contribution. The goal is not to make every high-dimensional problem analytic, but to identify a verifiable regime in which exact or controlled monotone transport follows from geometry and to state clearly when that regime is absent.

\section{Experimental Validation Protocol}
The present preprint contains the complete theory and a frozen evaluation protocol; we do not insert unpublished pilot measurements as paper results. The journal experiments are organized around theorem-level hypotheses rather than generic benchmark accuracy. Exact discrete OT is the ground truth for structural validation. Sinkhorn, sliced Wasserstein, and tree-sliced methods solve regularized or surrogate problems and are therefore used only as computational references, not as accuracy competitors for the original unregularized OT objective. Detailed code, data links, parameter grids, negative controls, and Go/No-Go rules are provided in the accompanying execution manual.

\subsection{E1: Exact Geometry and Mahalanobis Coordinates}
Synthetic $K$-chains verify Theorem 1 and Corollary 1 under both $M=I$ and non-diagonal $M=L^\top L$, with the cone transformed accordingly. Structural validation compares northwest-corner transport only against exact discrete OT. A separate efficiency layer reports full-cost construction, matrix-level Monge testing, Sinkhorn/FlashSinkhorn where available, SW, and tree-sliced references. Repeated-query experiments vary the number of transport queries $Q$ to test the amortized pair-evaluation claim in Proposition 9.

\subsection{E2: Exact Regret Geometry and Sharpness}
E2 is the primary theorem-facing experiment. Prescribed adjacent margins generate four defect topologies: single-cut, dense anti-Monge, clustered mixed-sign, and checkerboard mixed-sign. The first two instantiate equality cases of Corollary 2; the latter two expose interaction gaps between the true deficit polytope and its Fr\'echet box. Additional experiments verify the worst-case equality in Theorem 5, the Euclidean realizability in Corollary 3, and the strict-margin transition in Corollary 4. We report exact regret, the cumulative bound, interaction gap, cumulative-bound/regret ratio, distance-to-Monge on small matrices, distance-to-cone, $B_{\mathrm{geom}}/R_{\mathrm{obs}}$ when $R_{\mathrm{obs}}>0$, and $B_{\mathrm{geom}}/\osc(C)$. A formally valid but very loose bound is reported as conservative rather than presented as empirical confirmation.

\subsection{E3: High-Rank Real-Image Chains}
The exact theorem is tested in the original 3072-dimensional CIFAR-10 pixel space using non-collinear monotone chains. Starting from a real image, successive steps brighten distinct nonoverlapping patches with sufficient headroom, producing nonnegative increments whose numerical rank is reported explicitly. Two independently constructed chains remain ordered by $\R_+^d$, so the original Euclidean cross-cost must be Monge while the path explores many independent directions in the ambient image space. This experiment is a controlled high-dimensional stress test, not evidence that arbitrary natural images are ordered.

\subsection{E4: Held-Out Visual Order on Morpho-MNIST}
The primary Morpho-MNIST experiment uses the official rendering/downsampling pipeline rather than a construction that enforces the theorem. In E4-N, raw 784-dimensional pixels and a fixed 64-dimensional spatial ink-mass representation are evaluated without assuming coordinatewise monotonicity. The empirical question is whether local defect quantities predict the observed northwest-corner regret and whether the meaningful progression order is better than random or one-sided reversed controls. We report the cumulative and geometric certificates, $B_{\mathrm{geom}}/R_{\mathrm{obs}}$, $B_{\mathrm{geom}}/\osc(C)$, and rank correlations between structural scores and observed regret. Pixel polarity is determined from the released data rather than hard-coded as $1-$image.

A second, higher-value population-level stage treats order as a train-derived hypothesis. Candidate orders include paired-difference directions, PCA with orientation fixed on training data, ordinal/ridge scores, spectral seriation, random order, and one-sided target reversal. They are compared on validation by the regret-aligned score $U$ in (56); only a frozen candidate is evaluated on the held-out test. The primary selection diagnostic is the held-out association $\rho_S(U,R_{\mathrm{obs}})$ together with the regret separation from random and reversed controls. This stage is capped by a prespecified Go/No-Go gate and is not required for the validity of the main theory.

For numerical sanity only, E4-G retains a secondary controlled visual construction in which high-resolution binary dilations are pooled by a fixed positive linear operator. It verifies that a genuinely visual nested foreground transformation obeys the exact theorem, but it is not used as the headline evidence that natural rendered images are ordered; in the final journal version it is reported as extended validation rather than as the principal Morpho-MNIST result.

\subsection{E5: Optional Longitudinal Stress Test}
GRAPE is retained only as a one-day validation gate because uncontrolled longitudinal data need not exhibit ordered Monge structure \cite{r48}. If the gate fails, the branch stops and is treated as a boundary of applicability. Directional OT is not used as an accuracy competitor for northwest-corner transport: its displacement constraint $y-x\in K$ represents a different notion of cone structure from the within-support order studied here.

\begin{table*}[t]
\centering
\caption{Frozen evaluation logic. Exact discrete OT is the sole accuracy ground truth for the original Kantorovich objective; scalable OT variants are computational references because they solve regularized or surrogate problems.}
\label{tab:evaluation}
\small
\resizebox{\textwidth}{!}{%
\begin{tabular}{llll}
\toprule
Block & Data & Main question & Required evidence\\
\midrule
E1 & Synthetic exact chains & Does ground-space compatibility induce Monge costs and exact sparse transport? & Monge margins, exact-discrete-OT plan/cost gap, non-diagonal $M$, repeated-query pair evaluations and total runtime\\
E2 & Controlled defects & How do local defect topology and marginal feasibility determine global regret? & exact regret, box bound, equality cases, interaction gap, worst-case sharpness, stability radius\\
E3 & CIFAR-10 controlled chains & Does exactness persist in genuinely high-rank 3072-D image geometry? & increment rank, Monge margins, exact-discrete-OT exactness, end-to-end cost\\
E4 & Morpho-MNIST & Do natural rendered orders exhibit defect-controlled regret, and can $U$ select low-regret held-out orders? & natural-rendering regret, $B_{\mathrm{geom}}/\osc(C)$, $\rho_S(U,R_{\mathrm{obs}})$, frozen controls; guaranteed pooling only as secondary validation\\
E5 & GRAPE (gated) & Does structure survive uncontrolled longitudinal variation? & validation separation from controls; stop immediately if gate fails\\
\bottomrule
\end{tabular}}
\end{table*}

\section{Discussion and Limitations}
The strongest exact result requires an ordered support. The ambient geometry can be genuinely high-dimensional, but a chain still has a one-dimensional order coordinate. The approximate theory addresses violations of the Monge inequalities rather than eliminating this modeling requirement: Theorem 5 shows that the cumulative-defect bound is worst-case optimal when only local defect budgets are known, while Corollary 4 identifies a finite perturbation radius inside which exactness survives. Corollary 3 further shows that the worst-case mixed-difference patterns are not artifacts of arbitrary matrices, although they need not lie near a common acute cone. For broad partial orders, chain covers reveal block-Monge structure, yet mass allocation across chains remains a global problem unless additional assumptions are imposed.

The cone condition is likewise a modeling assumption rather than a universal property of data. A known order is appropriate when the application supplies a monotone factor; a train-derived order must be evaluated on held-out data and compared with simple alternatives. The experiments therefore include random, one-sided reversed, PCA, ordinal-score, and seriation controls. The official Morpho-MNIST rendering and the held-out order-selection stage provide the primary empirical stress tests; the guaranteed monotone-pooling trajectory is retained only as secondary controlled validation. If the natural rendering yields sizable local violations but the defect certificates still track the resulting regret, that supports the approximate theory rather than falsifying the geometry. If the structural scores and controls perform similarly and the certificates are uninformative, the correct conclusion is that the chosen representation does not support the ordered regime.

The computational gain is conditional. When the order is given and a reusable cone condition has already been verified, only the sparse support of the northwest-corner plan needs pairwise cost evaluation. When the order must be discovered from scratch, the cost of ordering, seriation, representation learning, or chain decomposition must be included. We report these components separately rather than attributing all preprocessing savings to the transport theorem.

Finally, baseline roles must be kept distinct. Exact discrete OT is the accuracy reference for the original Kantorovich objective. Sinkhorn solves an entropically regularized problem, while SW and tree-sliced methods define projected or surrogate discrepancies; they are therefore computational alternatives rather than like-for-like accuracy competitors. The present results are useful when preserving the original ground cost and exploiting a meaningful ordered structure are both part of the modeling objective.

\section{Conclusion}
We identified a direct route from high-dimensional cone geometry to classical Monge transportation. For squared Mahalanobis costs, acuteness of a cone is equivalent to the Monge property of every cross-cost matrix induced by two compatible chains, after which exact transport follows from classical northwest-corner optimality. The main quantitative contribution begins when this exact structure fails: northwest-corner regret admits an exact variational description over a marginal-dependent cumulative-deficit polytope, while its Fr\'echet-box relaxation converts local Monge defects into a global transport guarantee. The relaxation has an explicit equality characterization and is worst-case sharp under prescribed local defect budgets, with sharpness realizable by genuine squared-Euclidean cross-costs. Mahalanobis distance-to-cone deviations then translate geometric perturbations into local defects, and a strict angular margin yields a finite radius of exact stability. The resulting framework therefore explains both sides of ordered transport: when geometry makes the original high-dimensional problem exactly Monge, and how local departures from that geometry accumulate into global regret.

\balance

\appendices
\setcounter{theorem}{0}
\setcounter{proposition}{0}
\setcounter{corollary}{0}
\setcounter{lemma}{0}
\setcounter{definition}{0}
\setcounter{remark}{0}
\numberwithin{equation}{section}

\section{Cone--Monge Equivalence}\label{app:equivalence}
\begin{lemma}[Cross-difference identity]
Let $c_M(x,y)=(x-y)^\top M(x-y)$ with $M=M^\top$. For any $x_i,x_k,y_j,y_\ell$,
\begin{align}
&c_M(x_i,y_\ell)+c_M(x_k,y_j)-c_M(x_i,y_j)-c_M(x_k,y_\ell)\nonumber\\
&\qquad=2(x_k-x_i)^\top M(y_\ell-y_j).
\end{align}
\end{lemma}
\begin{proof}
Write $u=x_k-x_i$, $v=y_\ell-y_j$, and $z=x_i-y_j$. Then the left-hand side is
\[
(z-v)^\top M(z-v)+(z+u)^\top M(z+u)-z^\top Mz-(z+u-v)^\top M(z+u-v),
\]
and cancellation leaves $2u^\top Mv$.
\end{proof}

\begin{theorem}[Cone--Monge equivalence]
Let $K\subset\R^d$ be a closed convex cone and $M\succeq0$. The following are equivalent: (i) $u^\top Mv\ge0$ for all $u,v\in K$; (ii) every two finite $K$-chains induce a Monge cross-cost matrix under $c_M$.
\end{theorem}
\begin{proof}
Assume (i). If $i<k$ and $j<\ell$, then $x_k-x_i\in K$ and $y_\ell-y_j\in K$. The cross-difference identity is therefore nonnegative, which is exactly the Monge inequality. Conversely, fix arbitrary $u,v\in K$ and consider the two-point chains $(0,u)$ and $(0,v)$. The Monge inequality for their $2\times2$ cross-cost matrix gives $2u^\top Mv\ge0$.
\end{proof}

\begin{proposition}[Adjacent chain condition]
For fixed ordered sequences $x_1,\ldots,x_m$ and $y_1,\ldots,y_n$ with increments $\Delta x_i=x_{i+1}-x_i$ and $\Delta y_j=y_{j+1}-y_j$, the induced Mahalanobis cost matrix is Monge if and only if
\[
(\Delta x_i)^\top M(\Delta y_j)\ge0,\qquad i<m,\ j<n.
\]
\end{proposition}
\begin{proof}
The adjacent Monge margin is
\[
\gamma_{ij}=C_{i,j+1}+C_{i+1,j}-C_{ij}-C_{i+1,j+1}=2(\Delta x_i)^\top M(\Delta y_j).
\]
A matrix is Monge iff all adjacent margins are nonnegative.
\end{proof}

\begin{proposition}[Finite-generator characterization]
If $K=A\R_+^r$, then $K$ is $M$-acute if and only if $A^\top M A$ is entrywise nonnegative.
\end{proposition}
\begin{proof}
Every $u,v\in K$ can be written $u=A\alpha$, $v=A\beta$ for $\alpha,\beta\ge0$, hence $u^\top Mv=\alpha^\top(A^\top M A)\beta$. This is nonnegative for every nonnegative pair iff every entry of $A^\top M A$ is nonnegative.
\end{proof}

\section{High-Rank Chains and Coordinate Changes}
\begin{proposition}[High-rank cone chains]
For $K=\R_+^d$ and any $r\le\min\{d,m-1\}$ there is a $K$-chain $x_1,\ldots,x_m$ with increment rank at least $r$.
\end{proposition}
\begin{proof}
Take $x_1=0$ and choose the first $r$ increments to be the standard basis vectors $e_1,\ldots,e_r$. Choose every remaining increment in $\R_+^d$. All partial sums are ordered by $K$, while the first $r$ increments are linearly independent.
\end{proof}

\begin{proposition}[Euclidean representation of Mahalanobis acuteness]
If $M=L^\top L$ with $L$ invertible, then $K$ is $M$-acute iff $LK$ is Euclidean acute, and $c_M(x,y)=\|Lx-Ly\|_2^2$.
\end{proposition}
\begin{proof}
Both statements are immediate from $u^\top Mv=(Lu)^\top(Lv)$ and $(x-y)^\top M(x-y)=\|L(x-y)\|_2^2$.
\end{proof}

\section{Smooth Translation Costs}
\begin{proposition}
Let $c(x,y)=h(x-y)$ with $h\in C^2(\R^d)$ convex. If
\[
u^\top\nabla^2h(z)v\ge0,\qquad \forall z\in\R^d,\ u,v\in K,
\]
then every two $K$-chains induce a Monge cross-cost matrix.
\end{proposition}
\begin{proof}
Let $u=x_k-x_i\in K$, $v=y_\ell-y_j\in K$, and $z=x_i-y_j$. Put $F(s,t)=h(z+su-tv)$. Then
\[
\partial_{st}^2F(s,t)=-u^\top\nabla^2h(z+su-tv)v\le0.
\]
Thus $F$ is submodular on $[0,1]^2$ and
\[
F(0,0)+F(1,1)\le F(0,1)+F(1,0),
\]
which is the Monge inequality.
\end{proof}

\section{Exact Cumulative-Overlap Transport}
Let $A_i=\sum_{r\le i}a_r$, $B_j=\sum_{s\le j}b_s$, $A_0=B_0=0$, and define
\[
\omega_{ij}=\pos{\min(A_i,B_j)-\max(A_{i-1},B_{j-1})}.
\]
The matrix $\omega$ is the northwest-corner plan: it is the overlap length of the intervals $(A_{i-1},A_i]$ and $(B_{j-1},B_j]$. Hence it has marginals $a,b$. For a Monge cost matrix, classical Monge transportation theory implies that this plan is optimal. The result in the main paper follows immediately after applying the cone--Monge equivalence.

\section{Distance-to-Monge Regret}
\begin{theorem}
If $\|C-\widetilde C\|_\infty\le\delta$ for a Monge matrix $\widetilde C$, then the northwest-corner plan for probability marginals satisfies
\[
0\le\langle C,\pi_{\NW}\rangle-\OT_C\le2\delta.
\]
\end{theorem}
\begin{proof}
For every probability transport plan $\pi$,
\[
|\langle C-\widetilde C,\pi\rangle|\le\delta\sum_{ij}\pi_{ij}=\delta.
\]
Let $\pi^*$ be optimal for $C$. Since $\pi_{\NW}$ is optimal for $\widetilde C$,
\[
\langle C,\pi_{\NW}\rangle\le\langle\widetilde C,\pi_{\NW}\rangle+\delta\le\langle\widetilde C,\pi^*\rangle+\delta\le\langle C,\pi^*\rangle+2\delta.
\]
The lower bound is optimality of $\pi^*$.
\end{proof}

\section{Exact Regret Geometry and Defect Bounds}
For a coupling $\pi$ define
\[
H_\pi(i,j)=\sum_{r\le i,s\le j}\pi_{rs},\qquad i<m,\ j<n.
\]
\begin{lemma}[Discrete two-dimensional summation by parts]
If $P,Q$ have the same row and column sums, then
\[
\langle C,P-Q\rangle=-\sum_{i=1}^{m-1}\sum_{j=1}^{n-1}\gamma_{ij}(C)\bigl(H_P(i,j)-H_Q(i,j)\bigr).
\]
\end{lemma}
\begin{proof}
Let $E=P-Q$. Equal marginals imply zero row and column sums. Writing
\[
E_{ij}=H_E(i,j)-H_E(i-1,j)-H_E(i,j-1)+H_E(i-1,j-1)
\]
and summing $C_{ij}E_{ij}$ twice by parts leaves only the adjacent mixed differences of $C$, namely $-\gamma_{ij}(C)$, multiplied by $H_E(i,j)$.
\end{proof}

Let $A_i=\sum_{r\le i}a_r$, $B_j=\sum_{s\le j}b_s$, and define $D^\pi_{ij}=H_{\NW}(i,j)-H_\pi(i,j)$ and $\Dcal(a,b)=\{D^\pi:\pi\in\Pi(a,b)\}$.

\begin{theorem}[Exact regret representation]
For every cost matrix $C$ and probability marginals $a,b$,
\[
\langle C,\pi_{\NW}\rangle-\OT_C(a,b)=\max_{D\in\Dcal(a,b)}\left\{-\sum_{i,j}\gamma_{ij}(C)D_{ij}\right\}.
\]
\end{theorem}
\begin{proof}
For any feasible $\pi$, the summation-by-parts lemma with $P=\pi_{\NW}$ and $Q=\pi$ gives
\[
\langle C,\pi_{\NW}-\pi\rangle=-\sum_{i,j}\gamma_{ij}(C)D^\pi_{ij}.
\]
Taking the maximum over feasible $\pi$ is the same as subtracting the minimum transport cost from the northwest-corner cost, which proves the identity.
\end{proof}

Every feasible coupling satisfies the Fr\'echet--Hoeffding bounds, hence
\[
0\le D^\pi_{ij}\le w_{ij}:=\min(A_i,B_j)-\max\{0,A_i+B_j-1\}.
\]

\begin{theorem}[Cumulative-defect regret]
For every cost matrix $C$,
\[
0\le\langle C,\pi_{\NW}\rangle-\OT_C(a,b)\le\sum_{i,j}w_{ij}\pos{-\gamma_{ij}(C)}.
\]
\end{theorem}
\begin{proof}
The exact representation maximizes a linear form over $\Dcal(a,b)$. Relaxing $\Dcal(a,b)$ to the coordinatewise box $0\le D_{ij}\le w_{ij}$ gives the stated support function. Equivalently, terms with $\gamma_{ij}\ge0$ are maximized by $D_{ij}=0$ and terms with $\gamma_{ij}<0$ by $D_{ij}=w_{ij}$.
\end{proof}

\begin{corollary}[Equality characterization]
Equality in the cumulative-defect bound holds if and only if there exists a feasible coupling $\pi$ such that $D^\pi_{ij}=w_{ij}$ whenever $\gamma_{ij}(C)<0$ and $D^\pi_{ij}=0$ whenever $\gamma_{ij}(C)>0$. No restriction is imposed at zero margins.
\end{corollary}
\begin{proof}
The box support function is attained exactly on its coordinatewise maximizing face. Equality with the exact support function holds if and only if that face intersects $\Dcal(a,b)$.
\end{proof}

\begin{theorem}[Worst-case sharpness under local defect budgets]
Fix probability marginals $a,b$ and nonnegative $d_{ij}$. For
\[
\mathcal{C}(d)=\{C\in\R^{m\times n}:\gamma_{ij}(C)\ge-d_{ij}\},
\]
we have
\[
\sup_{C\in\mathcal{C}(d)}\{\langle C,\pi_{\NW}\rangle-\OT_C(a,b)\}=\sum_{i,j}w_{ij}d_{ij}.
\]
Consequently, each coefficient $w_{ij}$ is individually unimprovable if one only assumes the local constraints $\gamma_{ij}(C)\ge-d_{ij}$.
\end{theorem}
\begin{proof}
The upper bound follows from the cumulative-defect theorem. For the reverse inequality, construct a matrix with $\gamma_{ij}(C)=-d_{ij}$ for every adjacent rectangle. Such a matrix exists because adjacent mixed differences determine $C$ up to row and column potentials: set the first row and column to zero and define recursively
\[
C_{i+1,j+1}=C_{i,j+1}+C_{i+1,j}-C_{ij}+d_{ij}.
\]
Then $C$ is anti-Monge. The countermonotone Fr\'echet coupling has cumulative masses $H_{\mathrm{ct}}(i,j)=\max\{0,A_i+B_j-1\}$ and is optimal after reversing the target order. Summation by parts yields
\[
\langle C,\pi_{\NW}\rangle-\OT_C(a,b)=\sum_{i,j}d_{ij}(H_{\NW}-H_{\mathrm{ct}})=\sum_{i,j}w_{ij}d_{ij}.
\]
Adding row/column potentials or a global constant does not change regret.
\end{proof}

\begin{corollary}[Euclidean realizability of prescribed defects]
Let $D=(d_{ij})\in\R_+^{(m-1)\times(n-1)}$ and $r=\rank(D)$. There exist source and target sequences in $\R^r$ whose squared-Euclidean cross-cost matrix has adjacent margins $\gamma_{ij}=-d_{ij}$ for all $i,j$.
\end{corollary}
\begin{proof}
Take a rank factorization $-D/2=UV^\top$ with $U\in\R^{(m-1)\times r}$ and $V\in\R^{(n-1)\times r}$. Let the rows $u_i$ of $U$ and $v_j$ of $V$ be source and target increments, and define $x_1=y_1=0$, $x_{i+1}=x_i+u_i$, $y_{j+1}=y_j+v_j$. For squared Euclidean cost, the adjacent margin identity gives $\gamma_{ij}=2u_i^\top v_j=-d_{ij}$.
\end{proof}

\begin{corollary}[Single-cut sharpness]
For any fixed cut $(i_0,j_0)$ and $d>0$, there is a squared-Euclidean cross-cost matrix in finite dimension with $\gamma_{i_0j_0}=-d$ and all other adjacent margins zero for which the northwest-corner regret equals $w_{i_0j_0}d$.
\end{corollary}

\begin{corollary}[Uniform adjacent defect]
If $\gamma_{ij}(C)\ge-\varepsilon$ for every adjacent rectangle, then
\[
\langle C,\pi_{\NW}\rangle-\OT_C(a,b)\le\varepsilon\sum_{i,j}w_{ij}\le\frac{\varepsilon}{2}(m-1)(n-1).
\]
\end{corollary}

\begin{proposition}[Distance to the Monge cone from adjacent violations]
If $\gamma_{ij}(C)\ge-\varepsilon$ for every adjacent rectangle, then
\[
\delta_M(C)\le\frac{\varepsilon}{4}(m-1)(n-1).
\]
\end{proposition}
\begin{proof}
Let $r_i=i-(m+1)/2$, $s_j=j-(n+1)/2$, and set $\widetilde C_{ij}=C_{ij}-\varepsilon r_i s_j$. The adjacent margin of the correction is $+\varepsilon$, so $\widetilde C$ is Monge. Moreover $|r_i|\le(m-1)/2$ and $|s_j|\le(n-1)/2$, giving the claimed sup-norm distance.
\end{proof}

\section{Ordered Monge Support and Cyclical Monotonicity}
\begin{proposition}
If $C$ is Monge, the support of the northwest-corner plan is $c$-cyclically monotone.
\end{proposition}
\begin{proof}
Take any finite collection of support pairs ordered by increasing source index. The northwest-corner support is noncrossing, so their target indices are nondecreasing. Any cyclic reassignment can be decomposed into adjacent inversions of target indices. Replacing an inversion $(i,\ell),(k,j)$ with $i<k$ and $j<\ell$ by $(i,j),(k,\ell)$ does not increase cost by the Monge inequality. Repeatedly removing inversions returns the ordered assignment, proving that no cyclic rearrangement has smaller total cost.
\end{proof}

\section{Mahalanobis Geometric Perturbation Bound}
Assume $M\succ0$ and use the Hilbert norm $\|z\|_M=(z^\top Mz)^{1/2}$. For a closed convex cone $K$, let $P_K^M$ denote metric projection in this norm and define
\[
\eta^x_{i,M}=\dist_M(\Delta x_i,K),\qquad \eta^y_{j,M}=\dist_M(\Delta y_j,K).
\]
\begin{theorem}[Local geometric control]
If $K$ is $M$-acute, then
\[
\pos{-\gamma_{ij}(C)}\le2\bigl(\eta^x_{i,M}r^y_{j,M}+\eta^y_{j,M}r^x_{i,M}+\eta^x_{i,M}\eta^y_{j,M}\bigr),
\]
where $r^x_{i,M}=\|\Delta x_i\|_M$ and $r^y_{j,M}=\|\Delta y_j\|_M$. Consequently, weighting by the Fr\'echet widths gives the regret bound stated in the main paper.
\end{theorem}
\begin{proof}
Let $u_i=P_K^M(\Delta x_i)$, $v_j=P_K^M(\Delta y_j)$, and write $e_i=\Delta x_i-u_i$, $f_j=\Delta y_j-v_j$. Projection onto a closed convex set in a Hilbert space is nonexpansive and fixes $0\in K$, so
\[
\|u_i\|_M\le r^x_{i,M},\qquad \|v_j\|_M\le r^y_{j,M},
\]
while $\|e_i\|_M=\eta^x_{i,M}$ and $\|f_j\|_M=\eta^y_{j,M}$. Since $K$ is $M$-acute, $u_i^\top Mv_j\ge0$. By Cauchy--Schwarz in the $M$-inner product,
\begin{align*}
\Delta x_i^\top M\Delta y_j
&=u_i^\top Mv_j+e_i^\top Mv_j+u_i^\top Mf_j+e_i^\top Mf_j\\
&\ge-\eta^x_{i,M}r^y_{j,M}-\eta^y_{j,M}r^x_{i,M}-\eta^x_{i,M}\eta^y_{j,M}.
\end{align*}
The adjacent-margin identity $\gamma_{ij}=2\Delta x_i^\top M\Delta y_j$ yields the result.
\end{proof}

\begin{corollary}[Strict-margin stability]
Let reference increments $u_i,v_j\in K\setminus\{0\}$ satisfy
\[
\frac{u_i^\top Mv_j}{\|u_i\|_M\|v_j\|_M}\ge\kappa>0\qquad\text{for all }i,j.
\]
Suppose observed increments satisfy $\Delta x_i=u_i+e_i$ and $\Delta y_j=v_j+f_j$ with
\[
\|e_i\|_M\le\rho\|u_i\|_M,\qquad \|f_j\|_M\le\rho\|v_j\|_M.
\]
If $2\rho+\rho^2\le\kappa$, then every adjacent Monge margin of the observed cross-cost matrix is nonnegative. Hence the northwest-corner plan remains exactly optimal. Equivalently, exactness is guaranteed for $\rho\le\sqrt{1+\kappa}-1$.
\end{corollary}
\begin{proof}
By the assumed margin and Cauchy--Schwarz,
\begin{align*}
\Delta x_i^\top M\Delta y_j
&=u_i^\top Mv_j+e_i^\top Mv_j+u_i^\top Mf_j+e_i^\top Mf_j\\
&\ge(\kappa-2\rho-\rho^2)\|u_i\|_M\|v_j\|_M.
\end{align*}
The right-hand side is nonnegative under the stated condition. The adjacent cross-difference identity then gives $\gamma_{ij}\ge0$ for every $i,j$, so the cost matrix is Monge and the classical northwest-corner plan is optimal.
\end{proof}

\section{Continuous Chain-Supported Measures}
\begin{theorem}[Continuous ordered transport]
Let $I_X,I_Y$ be intervals, $\gamma_X:I_X\to\R^d$ and $\gamma_Y:I_Y\to\R^d$ Borel measurable injective $K$-monotone maps, and $\alpha,\beta$ probability measures on the intervals. Assume $c$ is lower semicontinuous, bounded below, Monge-compatible with $K$, and integrable under at least one coupling of $\mu=(\gamma_X)_\#\alpha$ and $\nu=(\gamma_Y)_\#\beta$. Then the coupling
\[
U\sim\mathrm{Unif}(0,1),\qquad X=\gamma_X(F_\alpha^{-1}(U)),\qquad Y=\gamma_Y(F_\beta^{-1}(U))
\]
is optimal for the Kantorovich problem between $\mu$ and $\nu$.
\end{theorem}
\begin{proof}
Define $\widetilde c(s,t)=c(\gamma_X(s),\gamma_Y(t))$. Monotonicity of the curves and compatibility of $c$ imply the quadrangle inequality
\[
\widetilde c(s_1,t_1)+\widetilde c(s_2,t_2)\le\widetilde c(s_1,t_2)+\widetilde c(s_2,t_1)
\]
whenever $s_1\le s_2$ and $t_1\le t_2$. Thus $\widetilde c$ is a one-dimensional Monge cost. The comonotone generalized-quantile coupling is optimal for such costs. Pushing the coupling forward by $(\gamma_X,\gamma_Y)$ preserves marginals and cost, which proves optimality in the original space. Atoms are handled by generalized quantiles in the usual way.
\end{proof}

\section{Chain Covers and a Restricted Block Model}
\begin{proposition}[Block-Monge property]
If source and target supports are decomposed into $K$-chains and $K$ is $M$-acute, every chain-pair cross-cost block is Monge.
\end{proposition}
\begin{proof}
Apply the cone--Monge theorem to each pair of chains.
\end{proof}

\begin{proposition}[Mixture-preserving decomposition]
Suppose $\mu=\sum_a\alpha_a\mu_a$ and $\nu=\sum_b\beta_b\nu_b$, where $\mu_a,\nu_b$ are probability measures. Let
\[
\PiMix=\left\{\sum_{a,b}q_{ab}\pi_{ab}:\ q\in\Pi(\alpha,\beta),\ \pi_{ab}\in\Pi(\mu_a,\nu_b)\right\}.
\]
Then, with $D_{ab}=\OT_c(\mu_a,\nu_b)$,
\[
\inf_{\pi\in\PiMix}\int c\,d\pi=\min_{q\in\Pi(\alpha,\beta)}\sum_{a,b}q_{ab}D_{ab}.
\]
\end{proposition}
\begin{proof}
For any fixed $q$, the objective separates over $(a,b)$ and each conditional coupling is minimized independently at cost $D_{ab}$. Minimizing the resulting linear objective over $q\in\Pi(\alpha,\beta)$ gives the result.
\end{proof}

\section{Boundary Regimes and $p$-Cost Extension}
\begin{proposition}[Solid-angle limitation]
If the direction of $Y-X$ is uniform on $S^{d-1}$, then $\Pr(X\preceq_KY)=\Omega(K)$. If $K$ is pointed, $K\cap(-K)$ has spherical measure zero and the two-way comparability probability is $2\Omega(K)$. For $K=\R_+^d$, symmetry of the $2^d$ orthants gives $\Omega(K)=2^{-d}$.
\end{proposition}
\begin{proof}
The event $X\preceq_KY$ is exactly $Y-X\in K$, which under a uniform direction has probability equal to the normalized spherical area of $K$. Pointedness makes $K$ and $-K$ disjoint up to a null boundary. Orthant symmetry gives the last identity.
\end{proof}

\begin{proposition}[Separable convex costs]
Let $c(x,y)=\sum_{r=1}^d h_r(x_r-y_r)$ with each $h_r$ convex. Then $K=\R_+^d$ is Monge-compatible with $c$.
\end{proposition}
\begin{proof}
For scalar $a\le b$ and $c\le d$, convexity of $h_r$ implies the one-dimensional Monge inequality
\[
h_r(a-c)+h_r(b-d)\le h_r(a-d)+h_r(b-c).
\]
Apply this coordinatewise to $x_1\le x_2$ and $y_1\le y_2$ and sum over $r$.
\end{proof}

\begin{proposition}[Scalar ordering does not imply Monge structure]
For $d\ge2$ and any nonzero linear score $s(x)=v^\top x$, there exist ordered source and target pairs under $s$ whose squared Euclidean cross-cost matrix is not Monge.
\end{proposition}
\begin{proof}
Rotate coordinates so $v=e_1$. Choose $u=(1,a,0,\ldots,0)$ and $w=(1,-a,0,\ldots,0)$ with $a>1$. Let $x_2-x_1=u$ and $y_2-y_1=w$. Then $v^\top u=v^\top w=1>0$, so both pairs are increasing under the score, while $u^\top w=1-a^2<0$. The cross-difference identity shows that the Monge inequality fails.
\end{proof}

\section{Displacement Interpolation on Exact Chains}
\begin{proposition}
Let $M\succ0$ and let $\pi^*$ be an optimal coupling for the quadratic Mahalanobis cost $c_M$. Define $\mu_t=((1-t)x+ty)_\#\pi^*$. Then $(\mu_t)_{t\in[0,1]}$ is a constant-speed geodesic for $W_{2,M}$. In particular, the statement applies to the northwest-corner coupling on compatible cone chains.
\end{proposition}
\begin{proof}
Choose an invertible matrix $L$ with $M=L^\top L$ and push the coupling forward by $(x,y)\mapsto(Lx,Ly)$. The transformed coupling is optimal for the ordinary squared Euclidean cost. Standard displacement interpolation in $W_2$ therefore gives
\[
W_2(L_\#\mu_s,L_\#\mu_t)=|t-s|W_2(L_\#\mu_0,L_\#\mu_1).
\]
Because $W_{2,M}(\alpha,\beta)=W_2(L_\#\alpha,L_\#\beta)$ and $L((1-t)x+ty)=(1-t)Lx+tLy$, pulling the path back yields the stated identity. For the cone-chain coupling, the explicit finite-sum formula follows directly from its support weights $\omega_{ij}$.
\end{proof}

\section{Amortized Pair-Evaluation Count}
\begin{proposition}
For a fixed compatible geometry and given chain orders, the northwest-corner plan for query $q$ with support sizes $m_q,n_q$ uses at most $m_q+n_q-1$ transported pairs. Hence $Q$ queries require at most $\sum_q(m_q+n_q-1)$ pairwise cost evaluations if costs are evaluated only on the plan support, whereas explicit formation of all cross-cost matrices requires $\sum_qm_qn_q$ evaluations.
\end{proposition}
\begin{proof}
The northwest-corner sweep exhausts at least one row or one column at each positive entry, so it creates at most $m_q+n_q-1$ nonzero entries. Summing over queries gives the first count; the second is the number of entries in the dense cross-cost arrays.
\end{proof}

\section{Auxiliary Ordered Objectives}\label{app:auxobj}
A proper cone $K$ and an order coordinate $\ell\in\operatorname{int}(K^*)$ that is injective on the support chain define a canonical chain parameterization through the scalar quantile of $\ell_\#\mu$. If $Q_\mu,Q_\nu$ are these canonical maps, then
\[
\operatorname{CCW}_{p,K}(\mu,\nu)=\left(\int_0^1\|Q_\mu(t)-Q_\nu(t)\|^p\,dt\right)^{1/p}
\]
is a metric on the canonical chain class by Minkowski's inequality and uniqueness of the canonical parameterization.

For directional tasks one may instead use
\[
\operatorname{DCOT}_{p,K}^p(\mu,\nu)=\inf_{\pi\in\Pi(\mu,\nu):\ y-x\in K\ \pi\text{-a.s.}}\int\|x-y\|^p\,d\pi.
\]
This is generally asymmetric and may be infinite. A soft version replaces hard feasibility by a penalty $\lambda\dist(y-x,K)^q$. These formulations are numerical extensions and are not part of the closed-form claim.

\section{Optional Learned-Cone Extension}\label{app:learnedcone}
Let $K_A=A\R_+^r$. A nondegenerate training objective can take the form
\begin{equation}
\widehat L(A)=\frac1N\sum_{t=1}^N\left[\dist(u_t,K_A)^2+\dist(v_t,K_A)^2\right]+\lambda\|\pos{-A^\top M A}\|_F^2+\beta L_{\mathrm{cov}}(A),
\end{equation}
with normalized generator columns and a pointedness/separation constraint. To prevent a degenerate solution that fits only a vanishing fraction of the observed increments, fix $\tau>0$ and a target coverage level $\rho_0\in(0,1)$ and define the soft empirical coverage
\begin{equation}
\widehat\rho_\tau(A)=\frac1{2N}\sum_{t=1}^N\left[\exp\left(-\frac{\dist(u_t,K_A)^2}{\tau^2}\right)+\exp\left(-\frac{\dist(v_t,K_A)^2}{\tau^2}\right)\right],
\end{equation}
with
\begin{equation}
L_{\mathrm{cov}}(A)=\pos{\rho_0-\widehat\rho_\tau(A)}^2.
\end{equation}
The coverage term is not an identifiability condition; it only rules out cones whose empirical support is too small to be useful. If $M$ is also learned, one should impose spectral bounds and a scale normalization such as $mI\preceq M\preceq MI$ and $\operatorname{tr}(M)=d$. The main paper does not rely on this learned-cone extension.

For any bounded loss class $\mathcal F=\{\ell_A:A\in\mathcal A\}$, standard Rademacher theory yields, with high probability,
\[
L(\widehat A)-\inf_{A\in\mathcal A}L(A)\le4\mathfrak R_N(\mathcal F)+O\left(\sqrt{\frac{\log(1/\delta)}{N}}\right).
\]
This is a structural-risk guarantee rather than a parameter-identifiability result.

\end{document}